\documentclass{esannV2}
\usepackage[dvips]{graphicx}
\usepackage[latin1,utf8]{inputenc}
\usepackage{amssymb,amsmath,array}

\usepackage{mathtools}
\usepackage{cleveref}
\usepackage{amsthm}
\usepackage{xcolor}
\usepackage{url}
\usepackage{cite}
\usepackage{comment}
\usepackage{enumitem}
\usepackage{amsmath,amsfonts,bm}

\def\eqref#1{equation~\ref{#1}}

\def\1{\bm{1}}

\def\vm{{\bm{m}}}

\def\mC{{\bm{C}}}

\def\mO{{\bm{O}}}

\def\mT{{\bm{T}}}

\def\mX{{\bm{X}}}

\DeclareMathAlphabet{\mathsfit}{\encodingdefault}{\sfdefault}{m}{sl}
\SetMathAlphabet{\mathsfit}{bold}{\encodingdefault}{\sfdefault}{bx}{n}

\newcommand{\E}{\mathbb{E}}

\newcommand{\R}{\mathbb{R}}

\newcommand{\Cov}{\mathrm{Cov}}

\newcommand*\diff{\mathop{}\!\mathrm{d}}

\newcommand{\dt}{\diff{t}}

\newcommand{\dWt}{\diff{B(t)}}

\newcommand{\postX}{\tilde{X}}

\crefname{equation}{eq.}{eq.}
\Crefname{equation}{Eq.}{Eq.}
\crefname{theorem}{thm.}{thms.}
\Crefname{Theorem}{Thm.}{Thms.}
\crefname{proposition}{prop.}{props.}
\Crefname{proposition}{Prop.}{Props.}
\crefname{definition}{dfn.}{dfn.}
\Crefname{definition}{Dfn.}{Dfn.}
\crefname{remark}{remark}{remark}
\Crefname{Remark}{Remark}{Remark}
\Crefname{algorithm}{Alg.}{Alg.}

\newtheorem{prop}{Proposition}
\newtheorem{dfn}{Definition}

\newcommand{\etal}{\textit{et al}. }
\newcommand{\ie}{\textit{i}.\textit{e}. }

\crefname{dfn}{Dfn.}{Dfns.}
\crefname{dfn}{Dfn.}{Dfns.}

\renewcommand{\paragraph}[1]{{\vspace{1mm}\noindent \bf #1}.}

\newcommand{\expE}[1]{\E\left[ #1 \right]}

\begin{document}
\title{Efficient Training of Neural SDEs Using Stochastic Optimal Control}

\author{
    Rembert Daems $^{1,2}$, Manfred Opper $^{3,4,5}$, \\
    Guillaume Crevecoeur $^{1,2}$ and Tolga Birdal $^{6}$
%
\thanks{
MO acknowledges funding by Deutsche Forschungsgemeinschaft (DFG)-SFB1294/ 1-318763901.
This research received funding from the Flemish Government under the "Onderzoeksprogramma Artificiële Intelligentie (AI) Vlaanderen" programme.
Furthermore it was supported by Flanders Make under the SBO project CADAIVISION. TB
was supported by a UKRI Future Leaders Fellowship [grant number MR/Y018818/1]. 
}
%
\vspace{.3cm}\\
%
1- D2Lab, Ghent University, Belgium
\vspace{.1cm}\\
2- FlandersMake@UGent -- corelab MIRO, Belgium
\vspace{.1cm}\\
3- Dept. of Theor. Comp. Science, Technical University of Berlin, Germany
\vspace{.1cm}\\
4- Inst. of Mathematics, University of Potsdam, Germany
\vspace{.1cm}\\
5- Centre for Sys. Modelling and Quant. Biomed., University of Birmingham, UK
\vspace{.1cm}\\
6- Dept. of Computing, Imperial College London, UK
\vspace{.1cm}\\
}

\maketitle

\begin{abstract}
We present a hierarchical, control theory inspired method for variational inference (VI) for neural stochastic differential equations (SDEs). While VI for neural SDEs is a promising avenue for uncertainty-aware reasoning in time-series, it is computationally challenging due to the iterative nature of maximizing the ELBO. In this work, we propose to decompose the control term into linear and residual non-linear components and derive an optimal control term for linear SDEs, using stochastic optimal control. Modeling the non-linear component by a neural network, we show how to efficiently train neural SDEs without sacrificing their expressive power. Since the linear part of the control term is optimal and does not need to be learned, the training is initialized at a lower cost and we observe faster convergence.
\end{abstract}
 
\section{Introduction}
Continuous-time models of dynamical systems provide a powerful framework for capturing the intricate variations in real-world phenomena. Among these, stochastic differential equations (SDEs) extend the capabilities of deterministic models by abstracting away unaccounted factors into instantaneous noise. SDEs naturally model various processes, including the motion of small particles (e.g., molecules) and financial market dynamics. When combined with neural networks~\cite{tzen2019neural,li2020scalable}, they become expressive tools for learning from irregular time-series observations.

Despite their promise, path-wise inference for neural SDEs remains a notorious challenge due to the complexity in fitting the non-Gaussian posterior distributions. Variational inference (VI) has become a prevalent tool with significant success in scaling inference methods~\cite{daems2024variational}. Yet, computational challenges persist. 

Existing works attempt to address these issues in VI for neural SDEs in various ways. Park~\etal~\cite{park2021neural} introduced finite-dimensional matching for efficient path comparison to train neural SDEs.
Kidger~\etal~\cite{kidger2021efficient} adopted a generative-adversarial approach to train these models.
Course and Nair~\cite{course2024amortized} proposed an amortized method for fast VI in latent neural SDEs, scaling efficiently with data size using a linear posterior. However, resorting to linear posteriors is a severe limitation in practice. 

Inspired by \emph{optimal control theory}, we propose a novel approach to efficiently perform VI in neural SDEs. Our key idea is to represent the prior as the combination of a linear model and a residual non-linear model. We leverage this decomposition to split the control function—used to compute the variational posterior—into two components. The first linear component is tractable and admits a closed-form solution, making it computationally efficient but less expressive.
The residual non-linear component, modeled by a neural network, captures higher-order effects at the cost of iterative optimization.
We combine the strengths of these two approaches. First, we compute the linear part in closed form, which serves as an efficient initialization for the neural network modeling the non-linear residual. This hierarchical design allows us to achieve faster and more stable inference compared to existing approaches that directly model the full control term~\cite{li2020scalable,daems2024variational}.

In summary, our contributions are:
\begin{enumerate}[leftmargin=\parindent,topsep=0.3mm,noitemsep]
    \item We derive the optimal control function solution for VI of a linear SDE driven by Brownian motion (BM), or by Markov--approximated fractional BM.
    \item We propose a neural SDE model with a linear and a residual non-linear (neural network) part, both for the prior SDE and the control terms, for which the linear part is optimal and does not need to be optimized or learned.
    \item We show that our proposed model trains faster and more stable than a standard non-linear network model on a financial data.
\end{enumerate}
We will make our implementation publicly available upon publication.

\section{Variational Inference of Stochastic Differential Equations}
\label{sec:vi-sde}
\begin{dfn}[SDEs driven by BM (BMSDE)]\label{dfn:BMSDE}
    A common generative model for stochastic dynamical systems considers a set of observational data ${\cal{D}} = \{O_1,\ldots, O_M\}$, 
where the $O_i$ are generated 
(conditionally) independent at random at discrete times $t_i$ with a likelihood $p_\theta\left(O_{i} \mid X(t_i)\right)$. The prior 
information about the unobserved path $\{X(t); t\in [0, \; T] \} $ of the latent process $X(t) \in \R^D$ is given by the assumption that
$X(t)$ fulfils the SDE: 
\begin{equation}\label{prior_process_Wiener}
\tag{\textsc{Prior}-SDE}
    \mathrm{~d} X(t) = b_\theta \left(X(t), t\right) \dt +\sigma_\theta\left(X(t), t\right) \dWt
\end{equation}
The \textit{drift function} $b_\theta \left(X(t), t\right)\in \R^D$  models the deterministic part of the change $\mathrm{~d} X(t)$ of the state variable 
$X(t)$ during the infinitesimal time interval $\dt$, whereas the \textit{diffusion matrix} $\sigma_\theta\left(X(t), t\right) \in \R^{D\times B}$
encodes the strength of the added Gaussian \textit{white noise} process, where $\dWt \sim {\cal{N}}(0,\diff t) \in \R^B$ is the infinitesimal increment of a vector of independent Wiener processes during $\dt$. 
\end{dfn}

\begin{dfn}[Posterior SDE]
The paths of the~\ref{prior_process_Wiener} can be steered by adding a control term $u(X(t), t)$
that depends on all variables to be optimised and the observations, to the drift resulting in the variational posterior~\cite{opper2019variational,li2020scalable}:
\begin{equation}\label{eq:bmsde-posterior}
    \diff \tilde{X}(t) = \left( b_\theta\left(\tilde{X}(t), t\right)+\sigma_\theta\left(\tilde{X}(t), t\right) u\left(\tilde{X}(t),t\right) \right) \dt+\sigma_\theta\left(\tilde{X}(t), t\right) \dWt
\end{equation}
\end{dfn}

In what follows, we will assume a parametric form for the control function $u(\tilde{X}(t), t) \equiv u_\phi(\tilde{X}(t), t)$ and will recall a scheme for inferring the \emph{variational parameters $(\theta,\phi)$}, \ie variational inference.

\begin{prop}[Variational Inference for BMSDE~\cite{opper2019variational,li2020scalable}]
    The \textit{variational parameters} $\phi$ are optimised by minimising the KL--divergence between the posterior and the prior, where the corresponding \textit{evidence lower bound} (ELBO) is maximized to find the most likely parameters $\theta$:
\begin{equation}
\label{eq:brownian-elbo}
\hspace{-0.2em}\sum_{i=1}^M\log p\left(O_i \mid \theta\right) \geq 
\mathbb{E}_{\postX}\left[\sum_{i=1}^M \log p_\theta\left(O_i \mid \postX(t_i) \right)-\int_0^T \frac{1}{2}\left\|u_\phi\left(\postX(t), t\right)\right\|^2\dt\right]
\end{equation}
where the observations $\{O_i\}$ are included by likelihoods $p_\theta\left(O_{i} \mid \postX(t_i)\right)$ and the expectation is taken over random paths of the approximate posterior process defined by 
(\cref{eq:bmsde-posterior}).
\end{prop}

\section{Optimal Control for Variational Inference for SDEs}
Our approach uses optimal control to decouple the possible linear and non-linear effects in the drift. While the linear part is easier to solve in closed-form, the non-linear terms will account for the complex variations in real data. In the sequel, we describe these two parts, respectively, finally leveraging the strengths of both.
\subsection{Optimal posterior control term for a linear prior SDE}
\label{sec:optimal-control}
The control term $u(x,t)\vcentcolon= u_\phi(x,t)$ can be obtained explicitly from the solution of the transformed \emph{Hamilton--Jacobi--Bellman} equation (HJBE)~\cite{kappen2005linear,archambeau2011approximate,maoutsa2022deterministic}:
\begin{equation}
\label{eq:exact_control}
u(x,t)=\sigma_\theta(x, t)^\top \nabla_x \log \E_\text{prior} \left[\prod_{i: t_i > t} p_\theta\left(O_i \mid X(t_i) \right) | X(t) = x\right].
\end{equation}
In general, such expectations over the paths of the \ref{prior_process_Wiener} involve solving second order partial differential equations in the $D+1$ variables and are intractable in closed form. However, in what follows, we will show how to compute it exactly when both the prior process $X(t)$ and the observation likelihood are Gaussian. This requires the drift 
$b_\theta \left(x, t\right)$ to be a linear function in $x$, and the diffusion $\sigma_\theta(t)$ independent of $x$.

\begin{prop}
    For a process $X(t)$ with linear drift and state-independant diffusion $\sigma(t)$ where we have $M$ observations 
    $\mO = [O(T_1), \dots, O(T_M)]$ after time $t$, the optimal control term takes the form:
    \begin{align}
    u(x, t) &= \sigma(t)^\top\nabla_x \log {\cal{N}}(\mO;\vm_x,\mC + \bm{\Sigma}_0) \\
    &= \sigma(t)^\top\left(\nabla_x \vm_x\right)^\top (\mC+\bm{\Sigma}_0)^{-1} \left( \mO - \vm_x \right),
    \label{eq:optimal-control-term}
    \end{align}
    where $p(\mX(\mT)|x)  = {\cal{N}}(\vm_x, \mC)$ is the joint Gaussian distribution of the
    solutions of the prior SDE $\mX(\mT)=[X(T_1), \dots, X(T_M)]$ conditioned on $X(t) =x$ having 
    mean vector by $\vm_x$ and covariance matrix by $\mC$.
    The observation likelihood is assumed to be of the form ${\cal{N}}(\mO;0, \bm{\Sigma}_0)$.
\end{prop}
\begin{proof}[Sketch of the proof]
    Under these assumptions, the expectation in \cref{eq:exact_control} for $X(t)$ becomes an $M$ dimensional 
    Gaussian integral of the form:
    \begin{equation}
        \E_\text{prior} \left[ \ldots \right]  = \int p(\mX(\mT)|x) p(\mO| \mX(\mT)) \diff \mX(\mT) = {\cal{N}}(\mO;\vm_x,\mC + \bm{\Sigma}_0) \, .
    \end{equation}
\end{proof}
Specifically, for a one--dimensional process $X(t)\in \mathbb{R}$ parameterized by $\lambda \in \mathbb{R}_+$, $\eta \in \mathbb{R}$ and constant diffusion $\varsigma \in \mathbb{R}_+$:
\begin{equation}
    \diff X(t) = \left( - \lambda X(t) + \eta\right) \dt + \varsigma \diff B(t)
\end{equation}
we can write the solution at some later time $T$ conditioned on the state $x$ at current time $t$ as~\cite{sarkka2019applied}:
\begin{equation}
    X(T) = x e^{-\lambda(T-t)} + \int_t^T e^{-\lambda(T-s)} \eta \diff s + \int_t^T e^{-\lambda(T-s)} \varsigma \diff B(s)
\end{equation}
which leads to the mean and covariance:
\begin{align}
    \label{eq:linear-prior-mean}
    {\vm_x}_{(i)} &= \expE{X(T_i | X(t) =x} = x e^{-\lambda(T_i-t)} + \frac{\eta}{\lambda}\left( 1 - e^{-\lambda(T_i-t)} \right) \\
    \label{eq:linear-prior-cov}
    \mC_{(i,j)} &= \Cov\left(X(T_i),X(T_j) \right) = \varsigma^2 \int_t^{\min(T_i,T_j)} e^{-\lambda(T_i-s)}e^{-\lambda(T_j-s)} \diff s \\
    \label{eq:linear-prior-cov-bis}
    &= \varsigma^2 \frac{e^{-\lambda |T_i - T_j|} - e^{-\lambda(T_i+T_j-2t)}}{2 \lambda} \, .
\end{align}

\begin{figure}
    \centering
    \includegraphics[width=\textwidth]{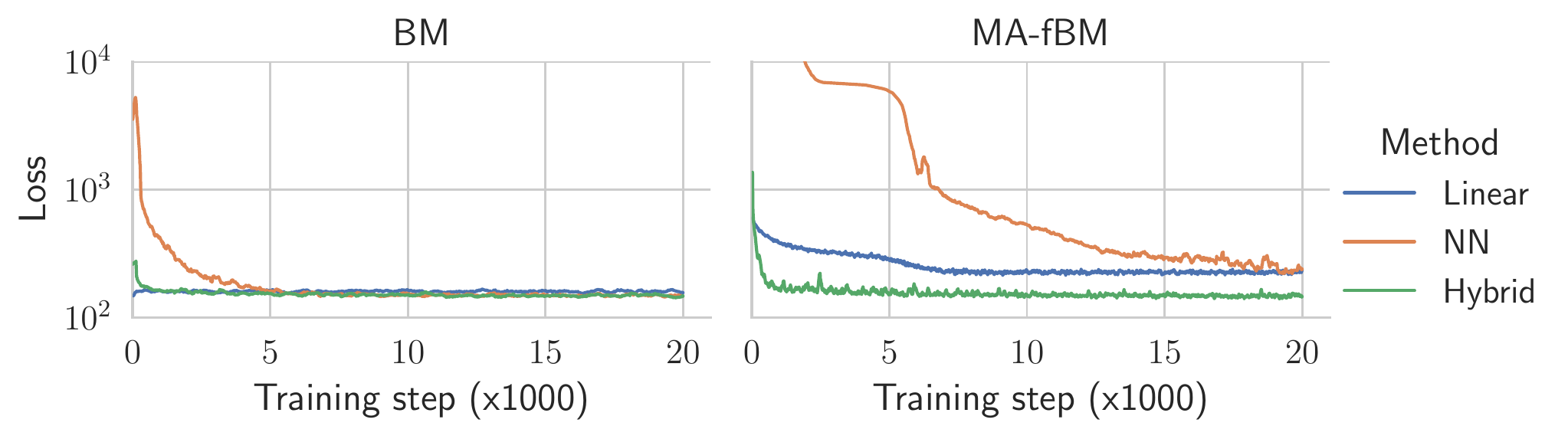}
    \caption{
        We show the loss (negative ELBO) curves of the models driven by BM (left) and MA-fBM (right).
        For both experiments, our proposed hybrid model (green) starts training with a loss that is multiple orders of magnitude smaller and converges much faster than a standard non-linear neural network model (blue).
        Our hybrid model (green) also performs better than the strictly linear model (orange), especially for the MA-fBM experiment.
    }
    \label{fig:loss-curves}
\end{figure}
\subsection{Incorporating non-linear residual terms}
\label{sec:hybrid-model}
We propose to define a prior SDE composed of linear and non-linear drifts as
\begin{equation}
    \diff X(t) = \left(- \lambda_\theta X(t) + \eta_\theta + b_\theta\left(X(t)\right) \right) \dt + \left( \varsigma_\theta + \sigma_\theta(X(t)) \right) \dWt
\end{equation}
where $b_\theta(\cdot)$ and $
\sigma_\theta(\cdot)$ are non-linear functions (e.g. neural networks) and $\theta$ indicates learnable parameters.
Equivalently, the control term is defined as
\begin{equation}
    u(\tilde{X}(t), t) \equiv u_c(\tilde{X}(t), t) + u_\phi(\tilde{X}(t), t)
\end{equation}
where $u_c(\cdot)$ is the analytical optimal control solution (\cref{eq:optimal-control-term}) that depends on $\lambda_\theta$, $\eta_\theta$ and $\varsigma_\theta$ (\cref{eq:linear-prior-mean,eq:linear-prior-cov-bis}) and $u_\phi$ is a residual non-linear control term,
modeled e.g. by a neural network.
For a purely linear model, without the non-linear components, the ELBO would be optimal by definition.
However, such a model would not be expressive, i.e., not be able to capture realistic, non-linear data.
The core idea of our work is to combine the linear terms with the residual
non-linear terms $b_\theta(\cdot)$, $\sigma_\theta(\cdot)$ and $u_\phi(\cdot)$ such that
the training of the model is more robust and fast, benefiting from the best of both worlds.

Furthermore, a crucial advantage of the linear model is the use of the tractable log-likelihood function $\log {\cal{N}}(\mO;\vm_x,\mC + \bm{\Sigma}_0)$
to directly find $\lambda_\theta$, $\eta_\theta$ and $\varsigma_\theta$, without having to solve computationally costly SDEs.
This allows initialization of training where the linear component is already optimal.

\paragraph{Extension to fractional Brownian motion (fBM)}
A method for variational inference for SDEs with long-term correlation, driven by fBM, was recently proposed~\cite{daems2024variational}.
A Markov approximation of fBM (MA-fBM) is used, essentially enlarging the state--space by multiple processes driven by a shared BM. This allows
variational inference in a similar way as explained in~\Cref{sec:vi-sde}. Hence, SDEs driven by MA-fBM readily benefits from our proposed methods.

\section{Experiments}
We apply our method on the first 500 days of the 3--Month US Treasury Bills\footnote{\url{https://fred.stlouisfed.org/series/DTB3}}.
We compare the training of our proposed hybrid model with the non-linear residual part to the training of a standard non-linear model and a strictly linear model. We also apply our method to the SDEs driven by
MA-fBM, presented in~\Cref{sec:hybrid-model}.
The non-linear prior drift $b_\theta(\cdot)$, diffusion $\sigma_\theta(\cdot)$ and control term $u_\phi(\cdot)$ are
neural networks.
The observations are encoded by an additional neural network into $u_\phi(\cdot)$, as is typically done in VI for SDEs~\cite{li2020scalable,daems2024variational}.
All neural networks have three layers, $128$ hidden neurons and the $\tanh$ activation function.
The observations noise $\bm{\Sigma}_0 = 0.1^2 \mathbf{I}$.
For the MA-fBM experiment we set a Hurst index of $0.65$
which is a reasonable choice for this data~\cite{lysy2013statistical}.
\Cref{fig:loss-curves} shows the loss (negative ELBO) curves of the three models, both for the models driven by BM and MA-fBM.

\section{Conclusion}
We present an optimal control inspired method for efficient variational inference for (neural) SDEs.
Under practically reasonable assumptions, we explicitly formulate the control term with linear and residual non-linear components
and derive a closed-form control term for the linear part using stochastic optimal control.
This model is shown to converge faster than a standard non-linear SDE, both for SDEs driven by BM and Markov-approximate fBM.

\paragraph{Future work and limitations} Our work applies only to 1-d SDEs, future work will involve a multi-dimensional formulation. We also plan to cover latent SDEs~\cite{course2024amortized}.

\begin{footnotesize}

\bibliographystyle{unsrt}
\bibliography{references}

\end{footnotesize}

\end{document}